\documentclass[11pt]{article}
\usepackage{graphicx}
\usepackage{amsmath,amsthm,amscd}
\usepackage{amssymb}
\usepackage{fancybox}
\usepackage{latexsym}
\usepackage{fancyhdr}
\usepackage{lastpage}
\usepackage{slashbox}
\usepackage{url}
\usepackage{lscape}
\usepackage{multirow}
\usepackage{mlapa}

\newtheorem{theorem}{Theorem}

\newtheorem{remark}{Remark}
\newtheorem{example}{Example}

\usepackage{bm}
\newcommand{\alphab}{{\bm \alpha}}
\newcommand\cp{\stackrel{p}{\longrightarrow}} 
\newcommand{\balpha}{\mbox{\boldmath $\alpha$}}
\newcommand{\alphabold}{\mbox{\boldmath $\alpha$}}
\newcommand{\simiid}{{\sim_\iid}}
\newcommand{\betabold}{\mbox{\boldmath $\beta$}}
\newcommand{\phib}{\bm \phi}

\newcommand{\bphi}{\bm \phi}
\newcommand{\Rbb}{\mathbb{R}}
\newcommand{\Real}{\Rbb}
\newcommand{\Cov}{{\rm Cov}}
\newcommand{\iid}{{\it i.i.d.\,}}
\newcommand{\0}{\mbox{\bf 0}}
\newcommand{\1}{\mbox{\bf 1}}
\renewcommand{\b}{{\bm b}}
\renewcommand\u{{\bm u}}
\newcommand{\x}{{\bm x}}
\newcommand{\etab}{{\bm \eta}}
\newcommand{\btheta}{{\bm \theta}}
\newcommand{\boldeta}{{\bm \eta}}
\newcommand{\thetab}{{\bm \theta}}

\begin{document}
%
\title{Semi-Supervised learning with Density-Ratio Estimation}
\author{
  Masanori Kawakita\\ Kyushu University\\ \tt{kawakita@inf.kyushu-u.ac.jp}
  \and
  Takafumi Kanamori\\ Nagoya University \\ \tt{kanamori@is.nagoya-u.ac.jp}
 }
%
%
%
%
\date{}
\maketitle

\begin{abstract}
 In this paper, we study statistical properties of semi-supervised learning, which is
 considered as an important problem in the community of machine learning. 
 In the standard supervised learning, only the labeled data is observed. 
 The classification and regression problems are formalized as the supervised learning. 
 In semi-supervised learning, unlabeled data is also obtained in addition to labeled
 data. 
 Hence, exploiting unlabeled data is important to improve the prediction accuracy 
 in semi-supervised learning. 
 This problems is regarded as a semiparametric estimation problem with missing data. 
 Under the the discriminative probabilistic models, 
 it had been considered that the unlabeled data is useless to improve the
 estimation accuracy. 
 Recently, it was revealed that the weighted estimator 
 using the unlabeled data achieves better prediction accuracy in comparison to the 
 learning method using only labeled data, especially when the discriminative probabilistic
 model is misspecified. 
 That is, the improvement under the semiparametric model with missing data is possible,
 when the semiparametric model is misspecified. 
 In this paper, we apply the density-ratio estimator to obtain the weight function 
 in the semi-supervised learning. 
 The benefit of our approach is that the proposed estimator does not require
 well-specified probabilistic models for the probability of the unlabeled data. 
 Based on the statistical asymptotic theory, we prove that the estimation accuracy of our
 method outperforms the supervised learning using only labeled data. 
 Some numerical experiments present the usefulness of our methods. 
\end{abstract}

\section{Introduction}
\label{sec:Introduction}
In this paper, we analyze statistical properties of semi-supervised learning. 
In the standard supervised learning, only the labeled data $(x,y)$ is observed, and 
the goal is to estimate the relation between $x$ and $y$. 
In semi-supervised learning \cite{chapelle06:_semi_super_learn}, 
the unlabeled data $x'$ is also obtained in addition to labeled data. 
In real-world data such as the text data, we can often obtain both labeled and unlabeled
data. 
A typical example is that $x$ and $y$ stand for the text of an article, and the tag of the
article, respectively. Tagging the article demands a lot of effort. Hence, 
the labeled data is scarce, while the unlabeled data is abundant. 
In semi-supervised learning, studying methods of exploiting unlabeled data is an important
issue. 

In the standard semi-supervised learning, 
statistical models of the joint probability $p(x,y)$, i.e., generative models, are 
often used to incorporate the information involved in the unlabeled data into the
estimation. 
For example, under the statistical model $p(x,y;\betabold)$ having the parameter
$\betabold$, the information involved in the unlabeled data is used to estimate the
parameter $\betabold$ via the marginal probability
$p(x;\betabold)=\int{}p(x,y;\betabold)dy$. The amount of information in unlabeled samples
is studied by 
\cite{castelli96,dillon10:_asymp_analy_gener_semi_super_learn,sinha07:_value_label_unlab_examp_model_imper}. 
This approach is developed to deal with a various data structures. 
For example, semi-supervised learning with manifold assumption or cluster assumption
has been studied along this line 
\cite{belkin04:_semi_super_learn_rieman_manif,DBLP:conf/nips/LaffertyW07}. 
Under some assumptions on generative models, 
it is revealed that unlabeled data is useful to improve the prediction accuracy. 

Statistical models of the conditional probability $p(y|x)$, 
i.e., discriminative models, are also used in semi-supervised learning. 
It seems that the unlabeled data is not useful 
that much for the estimation of the conditional probability, since the marginal
probability does not have any information on $p(y|x)$ 
\cite{lasserre06:_princ_hybrid_gener_discr_model,seeger01:_learn,zhang00}. 
Indeed, the maximum likelihood estimator using a parametric model of $p(y|x)$ 
is not affected by the unlabeled data. 
Sokolovska, et al.~\cite{sokolovska08}, however, proved that 
even under discriminative models, 
unlabeled data is still useful to improve the prediction accuracy of the learning method
with only labeled data. 

Semi-supervised learning methods basically work well under some assumptions on 
the population distribution and the statistical models. 
However, it was also reported that the semi-supervised learning has a possibility 
to degrade the estimation accuracy, especially when a misspecified model is applied
\cite{cozman03:_semi,grandvalet05:_semi,nigam99:_text_class_label_unlab_docum_em}. 
Hence, a \emph{safe} semi-supervised learning is desired. 
The learning algorithms proposed by Sokolovska, et al.~\cite{sokolovska08} and Li and Zhou
\cite{li11:_towar_makin_unlab_data_never_hurt} 
have a theoretical guarantee such that 
the unlabeled data does not degrade the estimation accuracy. 

In this paper, we develop the study of \cite{sokolovska08}. To incorporate the information
involved in unlabeled data into the estimator, Sokolovska, et al.~\cite{sokolovska08}
used the weighted estimator. In the estimation of the weight function, a well-specified
model for the marginal probability $p(x)$ was assumed. 
This is a strong assumption for semi-supervised learning. 
To overcome the drawback, we apply the density-ratio estimator for the estimation of the
weight function 
\cite{sugiyama12:_machin_learn_non_station_envir,sugiyama12:_densit_ratio_estim_machin_learn}. 
We prove that the semi-supervised learning with the density-ratio estimation 
improves the standard supervised learning. 
Our method is available not only classification problems but also regression problems,
while many semi-supervised learning methods focus on binary classification problems. 

This paper is organized as follows. 
In Section \ref{sec:problem_setup}, we show the problem setup. 
In Section \ref{sec:Inference_using_Weighted_Estimator}, 
we introduce the weighted estimator investigated by Sokolovska,~et
al.,\cite{sokolovska08}. 
In Section \ref{sec:Density-ratio_estimation}, we briefly explain the density-ratio
estimation. 
In Section \ref{sec:Semi-Supervised_Learning_with_Density-Ratio_Estimation}, 
the asymptotic variance of the estimators under consideration is studied. 
Section \ref{sec:Maximum_Improvement_SSL} is devoted to prove that the weighted estimator
using labeled and unlabeled data outperforms the supervised learning using only labeled data. 
In Section \ref{sec:Numerical_Experiments}, numerical experiments are presented. 
We conclude in Section \ref{sec:Conclusion}.

\section{Problem Setup}
\label{sec:problem_setup}
We introduce the problem setup. 
We suppose that the probability distribution of training samples is given as 
\begin{align}
 (x_i,y_i)\simiid{}p(y|x){}p(x),\ \ i=1,\ldots.n,
 \qquad
 x_j'\simiid{}q(x),\ \ j=1,\ldots,n', 
 \label{eqn:problem-setup}
\end{align}
where $p(y|x)$ is the conditional probability of $y\in\mathcal{Y}$ given
$x\in\mathcal{X}$, and  
$p(x)$ and $q(x)$ are the marginal probabilities on $\mathcal{X}$. 
Here, $q(x)$ is regarded as the probability in the testing phase, i.e., the test data
$(x,y)$ is distributed from the joint probability $p(y|x)q(x)$, and the estimation
accuracy is evaluated under the test probability. The paired sample $(x_i,y_i)$ is called
``labeled data'', and the unpaired sample $x_j'$ is called ``unlabeled data''. 
Our goal is to estimate the conditional probability $p(y|x)$
or the conditional expectation $E[y|x]$ 
based on the labeled and unlabeled data in \eqref{eqn:problem-setup}. 
When $\mathcal{Y}$ is a finite set, the problem is called the classification problem. 
For $\mathcal{Y}=\Rbb$, the estimation of $E[y|x]$ is referred to as the regression
problem. 

We describe the assumption on the marginal distributions, $p(x)$ and $q(x)$ in 
\eqref{eqn:problem-setup}. 
In the context of the \emph{covariate shift} adaptation \cite{JSPI:Shimodaira:2000}, 
the assumption that $p(x)\neq{}q(x)$ is employed in general. 
The weighted estimator with the weight function $q(x)/p(x)$ is used to correct the
estimation bias induced by the covariate shift; see 
\cite{sugiyama12:_machin_learn_non_station_envir,sugiyama12:_densit_ratio_estim_machin_learn}
for details. 
Hence, the estimation of the weight function $q(x)/p(x)$ is important to achieve a good
estimation accuracy. 
On the other hand, in the \emph{semi-supervised learning} \cite{chapelle06:_semi_super_learn}, 
the equality $p(x)=q(x)$ is assumed, and often $n'$ is much larger than $n$. 
This setup is also quite practical. 
For example, in the text data mining, the labeled data is scarce, while the unlabeled data
is abundant. 
In this paper, we assume that the equality 
\begin{align}
 p(x)=q(x) 
 \label{eqn:p_eq_q}
\end{align}
holds. 

We define the following semiparametric model, 
\begin{align}
 \mathcal{M}=\left\{p(y|x;\balpha)r(x)\,:\,
 \balpha\in{A}\subset\Rbb^d,\ r\in\mathcal{P}\right\}, 
 \label{eqn:para-model-cond-density}
\end{align}
for the estimation of the conditional probability $p(y|x)$, where $\mathcal{P}$ is the set
of all probability densities of the covariate $x$. 
The parameter of interest is $\alphab$, and $r(x)\in\mathcal{P}$ is the nuisance
parameter. 
The model $\mathcal{M}$ does not necessarily include the true probability $p(y|x)q(x)$, 
i.e., there may not exist the parameter $\alphab$ such that $p(y|x)=p(y|x;\alphab)$ holds. 
This is the significant condition, 
when we consider the improvement of the inference with the labeled and unlabeled data. 
Our target is to estimate the parameter $\alphab^*$ satisfying 
\begin{align}
 \max_{\alphab\in{A}}E[\log{}p(y|x;\alphab)]=E[\log{}p(y|x;\alphab^*)], 
 \label{eqn:true_parameter}
\end{align}
in which $E[\cdot]$ denotes the expectation with respect to the population distribution. 
If the model $\mathcal{M}$ includes the true probability, we have
$p(y|x;\alphab^*)=p(y|x)$ due to the non-negativity of Kullback-Leibler divergence
\cite{cover06:_elemen_of_infor_theor_wiley}. In the misspecified setup, however, the
equality $p(y|x;\alphab^*)=p(y|x)$ is not guaranteed.

\section{Weighted Estimator in Semi-supervised Learning}
\label{sec:Inference_using_Weighted_Estimator}
We introduce the weighted estimator. 
For the estimation of $p(y|x)$ under the model \eqref{eqn:para-model-cond-density}, 
we consider the maximum likelihood estimator (MLE). 
For the statistical model $p(y|x;\alphab)$, let $\u(x,y;\balpha)\in\Rbb^d$ be the score function
\begin{align*}
 \u(x,y;\alphab)=\nabla{}\log{}p(y|x;\alphab), 
\end{align*}
where $\nabla$ denotes the gradient with respect to the model parameter. 
Then, for any $\alphab\in{A}$, we have
\begin{align*}
 \int\u(x,y;\balpha) p(y|x;\balpha) p(x) dxdy=\0. 
\end{align*}
In addition, the extremal condition of \eqref{eqn:true_parameter} leads to 
\begin{align*}
 \int\u(x,y;\balpha^*) p(y|x) p(x) dxdy=\0. 
\end{align*}
Hence, we can estimate the conditional density $p(y|x)$ by $p(y|x;\widehat{\balpha})$, 
where $\widehat{\balpha}$ is a solution of the estimation equation 
\begin{align}
 \frac{1}{n}\sum_{i=1}^{n}\u(x_i,y_i;\balpha)=\0. 
 \label{eqn:Z-estimator}
\end{align}
Under the regularity condition, the MLE has the statistical consistency
to the parameter $\alphab^*$ in \eqref{eqn:true_parameter}; 
see \cite{vaart98:_asymp_statis} for details. 
In addition, the score function $\u$ is an optimal choice among Z-estimators
\cite{vaart98:_asymp_statis}, 
when the true probability density is included in the model $\mathcal{M}$. This implies
that the efficient score of the semiparametric model $\mathcal{M}$ is the same as the
score function of the model $p(y|x;\alphab)$. 
This is because, in the semiparametric model $\mathcal{M}$, 
the tangent space of the parameter of interest is orthogonal to that of the nuisance
parameter. 
Here, the asymptotic variance matrix of the estimated parameter 
is employed to compare the estimation accuracy. 


Next, we consider the setup of the semi-supervised learning. 
When the model $\mathcal{M}$ is specified, we find that the estimator
\eqref{eqn:Z-estimator} using only the labeled data is efficient. 
This is obtained from the results of numerous studies about the semiparametric 
inference with missing data; see \cite{nan09:_asymp,robins94:_estim} and references
therein. 

Suppose that the model $\mathcal{M}$ is misspecified. 
Then, it is possible to improve the MLE in \eqref{eqn:Z-estimator} by 
using the weighted MLE \cite{sokolovska08}. 
The weighted MLE is defined as a solution of the equation, 
\begin{align}
 \frac{1}{n}\sum_{i=1}^{n}w(x_i)\u(x_i,y_i;\balpha)=\0, 
 \label{eqn:wZ-estimator}
\end{align}
where $w(x)$ is a weight function. Suppose that $w(x)=q(x)/p(x)$. 
Then the law of large numbers leads to the probabilistic convergence, 
\begin{align*}
 \frac{1}{n}\sum_{i=1}^{n}w(x_i)\u(x_i,y_i;\balpha)
 \cp
 \int \frac{q(x)}{p(x)} \u(x,y;\alphabold) p(y|x)p(x)dx
 =
 \int \u(x,y;\alphabold) p(y|x)q(x)dx. 
\end{align*}
Hence the estimator $p(y|x;\widehat{\alphabold})$ based on \eqref{eqn:wZ-estimator}
will provide a good estimator of $p(y|x)$ under the marginal probability $q(x)$. 
This indicates that $p(y|x;\widehat{\alphabold})$ is expected to approximate $p(y|x)$ over
the region on which $q(x)$ is large. 
The weight function $w(x)$ has a role to adjust the bias of the estimator under the
covariate shift \cite{JSPI:Shimodaira:2000}. 
On the setup of the semi-supervised learning, however, $w(x)=q(x)/p(x)=1$ holds, and it is
known beforehand. 
Hence, one may think that there is no need to estimate the weight function. 
Sokolovska, et al.,\cite{sokolovska08} showed that estimation of the weight function is 
useful, even though it is already known in the semi-supervised learning. 

We briefly introduce the result in \cite{sokolovska08}. 
Let the set $\mathcal{X}$ be finite. Then, $\mathcal{P}$ is a finite dimensional
parametric model. 
Suppose that the sample size of the unlabeled data is enormous, and that 
the probability function $q(x)$ on $\mathcal{X}$ is known with a high degree of accuracy. 
The probability $p(x)$ is estimated by the maximum likelihood estimator 
$\widehat{p}(x)$ based on the samples $\{x_i\}_{i=1}^{n}$ in the labeled data. 
Then, Sokolovska, et al.~\cite{sokolovska08} showed that the weighted MLE
\eqref{eqn:wZ-estimator} with the estimated weight function $w(x)=q(x)/\widehat{p}(x)$ 
improves the naive MLE, when the model $\mathcal{M}$ is misspecified, i.e.,
$p(y|x)q(x)\not\in\mathcal{M}$. 

Shimodaira \cite{JSPI:Shimodaira:2000} pointed out that 
the weighted MLE using the exact density ratio $w(x)=q(x)/p(x)$ has the statistical
consistency to the target parameter $\alphab^*$, when the covariate shift occurs. 
Under the regularity condition, it is rather straightforward to see that the weighted MLE
using the estimated weight function $w(x)=q(x)/\widehat{p}(x)$ also converges to
$\alphab^*$ in probability, since $\widehat{p}(x)$ converges to $p(x)$ in probability. 
Sokolovska's result implies that when $p(x)=q(x)$ holds, 
the weighted MLE using the estimated weight function improves the weighted MLE using 
the true density ratio in the sense of the asymptotic variance of the estimator. 

The phenomenon above is similar to the statistical paradox analyzed by
\cite{henmi04,henmi07:_impor_sampl_via_estim_sampl}. 
In the semi-parametric estimation, Henmi and Eguchi \cite{henmi04} pointed out that 
the estimation accuracy of the parameter of interest can be improved by estimating the
nuisance parameter, even when the nuisance parameter is known beforehand. 
Hirano, et al.,~\cite{hirano03:_effic_estim_averag_treat_effec} also pointed out that 
the estimator with the estimated propensity score is more efficient than the estimator
using the true propensity score in the estimation of the average treatment effects. 
Here, the propensity score corresponds to the weight function $w(x)$ in our context. 
The degree of improvement is described by using the projection of the score function 
onto the subspace defined by the efficient score for the semi-parametric model. 
In our analysis, also the projection of the score function $\u(x,y;\alphabold)$ plays 
an important role as shown in Section \ref{sec:Maximum_Improvement_SSL}. 

For the estimation of the weight function in \eqref{eqn:wZ-estimator}, 
we apply the density-ratio estimator 
\cite{sugiyama12:_machin_learn_non_station_envir,sugiyama12:_densit_ratio_estim_machin_learn}
instead of estimating the probability densities separately. 
We show that the density-ratio estimator provides a practical method for the
semi-supervised learning. 
In the next section, we introduce the density-ratio estimation.

\section{Density-ratio estimation}
\label{sec:Density-ratio_estimation}
Density-ratio estimators are available to estimate the weight function $w(x)=q(x)/p(x)$. 
Recently, methods of the direct estimation for density-ratios have been developed in 
the machine learning community  
\cite{sugiyama12:_machin_learn_non_station_envir,sugiyama12:_densit_ratio_estim_machin_learn}. 
We apply the density-ratio estimator to estimate the weight function $w(x)$ 
instead of using the estimator of each probability density. 

We briefly introduce the density-ratio estimator according to 
\cite{Biometrika:Qin:1998}. 
Suppose that the following training samples are observed, 
\begin{align}
 x_i\simiid{}p(x),\ \ i=1,\ldots.n,
 \qquad
 x_j'\simiid{}q(x),\ \ j=1,\ldots,n'. 
 \label{eqn:training_data_densityratio}
\end{align}
Our goal is to estimate the density-ratio $w(x)=q(x)/p(x)$. 
The $r$-dimensional parametric model for the density-ratio is defined by 
\begin{align}
 w(x;\thetab)=\exp\{\theta_1\phi_1(x)+\cdots+\theta_r\phi_r(x)\}, 
 \label{eqn:density-ratio-model}
\end{align}
where $\phi_1(x)=1$ is assumed. 
For any function $\etab(x;\thetab)\in\Rbb^r$ which may depend on the parameter $\thetab$, 
one has the equality
\begin{align*}
 \int\etab(x;\thetab)w(x)p(x)dx-\int\etab(x;\thetab)q(x)dx=\0
\end{align*}
Hence, the empirical approximation of the above equation is expected to provide an
estimation equation of the density-ratio. 
The empirical approximation of the above equality 
under the parametric model of $w(x;\thetab)$ is given as
\begin{align}
 \frac{1}{n}\sum_{i=1}^{n}\boldeta(x_i;\btheta)w(x_i;\btheta)
-\frac{1}{n'}\sum_{j=1}^{n'}\boldeta(x_j';\btheta)=\0. 
 \label{eqn:density-ratio-estimator}
\end{align}
Let $\widehat{\thetab}$ be a solution of \eqref{eqn:density-ratio-estimator}, and then, 
$w(x;\widehat{\thetab})$ is an estimator of $w(x)$. 
Note that we do not need to estimate probability densities $p(x)$ and $q(x)$ separately. 
The estimation equation \eqref{eqn:density-ratio-estimator} provides a direct estimator of
the density-ratio based on the moment matching with the function $\etab(x;\thetab)$. 

Qin~\cite{Biometrika:Qin:1998} proved that the optimal choice of $\etab(x;\thetab)$ is
given as 
\begin{align}
 \etab(x;\thetab) 
 = \frac{1}{1+n'/n\cdot{}w(x;\thetab)}\nabla\log{}w(x;\thetab)
 = \frac{1}{1+n'/n\cdot{}w(x;\thetab)}\phib(x), 
 \label{eqn:optimal-density-ratio-est}
\end{align}
where $\phib(x)=(\phi_1(x),\ldots,\phi_r(x))^T$. 
By using $\etab(x;\thetab)$ above, the asymptotic variance matrix 
of $\widehat{\thetab}$ is minimized among the set of moment matching estimators, 
when $w(x)$ is realized by the model $w(x;\thetab)$. Hence, 
\eqref{eqn:optimal-density-ratio-est} is regarded as the counterpart of the score function
for parametric probability models.

\section{Semi-Supervised Learning with Density-Ratio Estimation}
\label{sec:Semi-Supervised_Learning_with_Density-Ratio_Estimation}

We study the asymptotics of the weighted MLE \eqref{eqn:wZ-estimator} using the
estimated density-ratio. The estimation equation is given as 
\begin{align}
 \left\{
 \begin{array}{ll}
  \displaystyle
  \frac{1}{n}\sum_{i=1}^{n}w(x_i;\btheta)\u(x_i,y_i;\balpha)=\0, \\
  \displaystyle
  \frac{1}{n}\sum_{i=1}^{n}\boldeta(x_i;\btheta)w(x_i;\btheta)
-\frac{1}{n'}\sum_{j=1}^{n'}\boldeta(x_j';\btheta)=\0. 
 \end{array}
\right.
 \label{eqn:Dress3}
\end{align}
Here, the statistical models \eqref{eqn:para-model-cond-density} and 
\eqref{eqn:density-ratio-model} are employed. 
The first equation is used for the estimation of the parameter $\alphab$ of the model
$p(y|x;\alphabold)$, and the second equation is used for the estimation of the
density-ratio $w(x;\thetab)$. 
The estimator defined by \eqref{eqn:Dress3} is refereed to as density-ratio estimation
based on semi supervised learning, or \emph{DRESS} for short. 

In Sokolovska, et al.\cite{sokolovska08}, 
the marginal probability density $p(x)$ is estimated
by using a well-specified parametric model. 
Clearly, preparing the well-specified parametric model is not practical, 
when $\mathcal{X}$ is not finite set. 
On the other hand, it is easy to prepare a specified model of the density-ratio $w(x)$, 
whenever $p(x)=q(x)$ holds in \eqref{eqn:problem-setup}. 
The model \eqref{eqn:density-ratio-model} is an example. Indeed, $w(x;\0)=1$ holds. 
Hence, the assumption that the true weight function is realized by 
the model $w(x;\thetab)$ is not of an obstacle in semi-supervised learning. 

We show the asymptotic expansion of the estimation equation \eqref{eqn:Dress3}. 
Let $\widehat{\balpha}$ and $\widehat{\btheta}$ be a solution of \eqref{eqn:Dress3}. 
In addition, define $\balpha^*$ be a solution of 
\begin{align*}
 \int\u(x,y;\balpha)p(y|x)p(x) dxdy=\0
\end{align*}
and $\thetab^*$ be the parameter such that $w(x;\thetab^*)=1$, i.e., $\thetab^*=\0$. 
We prepare some notations: 
$\u=\u(x,y;\alphab^*),\,\etab=\etab(x;\thetab^*),\,\u_i=\u(x_i,y_i;\alphab^*),\,\etab_i=\etab(x_i;\thetab^*),\,\etab_j'=\etab(x_j';\thetab^*)$,  
$\delta\alphab=\widehat{\alphab}-\alphab^*,\,\delta\thetab=\widehat{\thetab}-\thetab^*$. 
The Jacobian of the score function $\u$ with respect to the parameter $\alphab$ is denoted
as $\nabla\u$, i.e., the $d$ by $d$ matrix whose element is given as 
$(\nabla{\u}(x,y;\alphab))_{ik}=
\frac{\partial^2}{\partial\alpha_i\partial\alpha_k}\log{}p(y|x;\alphab)$. 
The variance matrix and the covariance matrix under the probability $p(y|x)p(x)$ 
are denoted as $V[\cdot]$ and $\Cov[\cdot,\cdot]$, respectively. 
Without loss of generality, we assume that $\etab$ at $\thetab=\thetab^*$ is represented
as 
\begin{align*}
 \etab(x;\thetab^*)=\phib(x)+\widetilde{\phib}(x), 
\end{align*}
where $\widetilde{\phib}(x)$ is an arbitrary function orthogonal to $\phib(x)$, 
i.e., $E[\phib\widetilde{\phib}^T]=O$ holds. If $\etab(x;\thetab^*)$ does not have 
any component which is represented as a linear transformation of $\phib(x)$, the estimator 
would be degenerated. 
Under the regularity condition, the estimated parameters, 
$\widehat{\alphab}$ and $\widehat{\thetab}$, 
converge to $\alphab^*$ and $\thetab^*$, respectively. 
The asymptotic expansion of \eqref{eqn:Dress3} around 
$(\alphab,\thetab)=(\alphab^*,\thetab^*)$ leads to 
\begin{align*}
 E[\nabla\u]\delta\alphab+E[\u\phib^T]\delta\thetab
 &=-\frac{1}{n}\sum_{i=1}^{n}\u_i+o_p(n^{-1/2}),\\
 E[\phib\phib^T]\delta\thetab
 &=
 \frac{1}{n'}\sum_{j=1}^{n'}\etab_j' -\frac{1}{n}\sum_{i=1}^{n}\etab_i+o_p(n^{-1/2}).
\end{align*}
Hence, we have 
\begin{align*}
 E[\nabla\u]\delta\alphab
 =
 \frac{1}{n}\sum_{i=1}^{n}
 \big\{E[\u\phib^T]E[\phib\phib^T]^{-1}\etab_i-\u_i \big\}
 -\frac{1}{n'}\sum_{j=1}^{n'} E[\u\phib^T]E[\phib\phib^T]^{-1}\etab_j'
 +o_p(n^{-1/2}). 
\end{align*}
Therefore, we obtain the asymptotic variance, 
\begin{align*}
 &\phantom{=}
 n\cdot{}E[\nabla\u]V[\delta\alphab]E[\nabla\u]^T\\
 &=
 V[\u]+
 \bigg(1+\frac{n}{n'}\bigg)
 E[\u\phib^T]E[\phib\phib^T]^{-1}V[\etab]E[\phib\phib^T]^{-1}E[\phib\u^T]\\
 &\phantom{=}
 - E[\u\phib^T]E[\phib\phib^T]^{-1}\Cov[\etab,\u]
 - \Cov[\u,\etab]E[\phib\phib^T]^{-1}E[\phib\u^T]+o(1)
\end{align*}
On the other hand, the variance of the naive MLE, $\widetilde{\alphab}$, defined as
a solution of \eqref{eqn:Z-estimator} is given as
\begin{align*}
 n\cdot{}E[\nabla\u]V[\delta\widetilde{\alphab}]E[\nabla\u]^T=V[\u]+o(1), 
\end{align*}
where $\delta\widetilde{\alphab}=\widetilde{\alphab}-\alphab^*$.

\section{Maximum Improvement by Semi-Supervised Learning}
\label{sec:Maximum_Improvement_SSL}
Given 
the model for the density-ratio
$w(x;\thetab)$, we compare the asymptotic variance matrices of the estimators, 
$\widetilde{\alphab}$ and $\widehat{\alphab}$. 
First, let us define 
\begin{align*}
 \bar{\u}(x)
 =\int\!\u(x,y;\alphab^*)p(y|x)dy, 
\end{align*}
i.e., $\bar{\u}(x)$ is the projection of the score function $\u(x,y;\alphab^*)$ onto the
subspace consisting of all functions depending only on $x$, where 
the inner product is defined by the expectation under the joint probability $p(y|x)p(x)$. 
Note that the equality $E[\bar{\u}]=\0$ holds. 
Let the matrix $B$ be 
\begin{align*}
B=E[\bar{\u}\phib^T]E[\phib\phib^T]^{-1}. 
\end{align*}
Then, a simple calculation yields that 
the difference of the variance matrix between $\widetilde{\alphabold}$ and
$\widehat{\alphabold}$ is equal to 
\begin{align}
 \mathrm{Diff}[\u]
 &
 :=n\cdot{}E[\nabla\u]V[\delta\widetilde{\alphab}]E[\nabla\u]^T
 -n\cdot{}E[\nabla\u]V[\delta{\alphab}]E[\nabla\u]^T \nonumber\\
 &=
 \frac{n'}{n+n'}E[\bar{\u}\bar{\u}^T]-\bigg(1+\frac{n}{n'}\bigg)V[B\etab-\frac{n'}{n+n'}\bar{\u}]
 +o(1). 
 \label{eqn:diff-bar_u-expression}
\end{align}
In the second equality, we supposed that $n'/n$ converges to a positive constant. 
When $\mathrm{Diff}[\u]$ is positive definite, the estimator $\widehat{\alphab}$ 
using the labeled and unlabeled data improves the estimator $\widetilde{\alphab}$ using
only the labeled data. 
It is straightforward to see that the improvement is not attained 
if $\bar{\u}=\0$ holds. 
In general, the score function $\u(x,y;\alphab)=\nabla\log{p}(y|x;\alphab)$ 
satisfies $\bar{\u}=\0$, if the model is specified. 
When the model of the conditional probability $p(y|x)$ is misspecified, 
however, there is a possibility that the proposed estimator \eqref{eqn:Dress3} outperforms
the MLE $\widetilde{\alphab}$. 

We derive the optimal moment function $\etab$ for the estimation of the parameter $\alphab^*$. 
The optimal $\etab$ can be different from
\eqref{eqn:optimal-density-ratio-est}. 
We prepare some notations. 
Let $\Pi_{\bm \phi}\bar{\u}$ be the $\Real^d$-valued function on $\mathcal{X}$, 
each element of which is the projection of each element of $\bar{\u}$ onto the subspace
spanned by $\{\phi_1(x),\ldots,\phi_r(x)\}$. 
Here, the inner product is defined by the expectation under the marginal probability $p(x)$. 
In addition, let $\Pi_{\bm\phi}^\bot\bar{\u}$ be the projection of $\bar{\u}$ onto the
orthogonal 
complement of the subspace, 
i.e., $\Pi_{\bm \phi}^\bot\bar{\u}=\bar{\u}-\Pi_{\bm \phi}\bar{\u}$. 

\begin{theorem}
 \label{theorem:optimal-moment-func}
 We assume that the model of the density-ratio is defined as 
 \begin{align*}
  w(x;\thetab)=\exp\{\bphi(x)^T\thetab\}
 \end{align*}
 with the basis functions $\phib(x)=(\phi_1(x),\ldots,\phi_r(x))$
 satisfying $\phi_1(x)=1$. Suppose that $E[\phib\phib^T]\in\Rbb^{r\times{r}}$ is
 invertible, and that the rank of $E[\bar{\u}\phib^T]E[\phib\phib^T]^{-1}$ is equal to the
 dimension of the parameter $\alphab$, i.e., row full rank. 
 We assume that the moment function $\etab(x;\thetab)$ at 
 $\thetab=\thetab^*$ is represented as 
\begin{align}
 \etab(x;\thetab^*)=\phib(x)+\widetilde{\phib}(x)
 \label{eqn:eta-representation}
\end{align}
 where 
 $\widetilde{\phib}(x)$ is a function orthogonal to
 $\phib(x)$, i.e., $E[\phib(x)\widetilde{\phib}(x)^T]=O$ holds. 
 Then, an optimal $\widetilde{\phib}$ is given as 
 \begin{align}
  \widetilde{\phib}=\frac{n'}{n+n'}B^T(BB^T)^{-1}\Pi_{\phib}^{\bot}\bar{\u}. 
  \label{eqn:optimal-eta}
 \end{align}
 For the optimal choice of $\etab$, the maximum improvement is given as 
 \begin{align}
  \mathrm{Diff}[\u]
  &=
  \frac{n'}{n+n'}E[\bar{\u}\bar{\u}^T]
  -\frac{n^2}{n'(n+n')}E[\Pi_{\phib}\bar{\u}(\Pi_{\phib}\bar{\u})^T]+o(1)\nonumber\\
  &=
   \frac{n'}{n+n'}E[{\Pi_{\phib}^\bot}\bar{\u}({\Pi_{\phib}^\bot}\bar{\u})^{T}]
  +\frac{n'-n}{n'}E[{\Pi_{\phib}}\bar{\u}({\Pi_{\phib}}\bar{\u})^{T}]+o(1)
  \label{eqn:maximum-improvement-opt}
 \end{align}
\end{theorem}
\begin{proof}
 Due to $\phi_1(x)=1$, one has $E[\widetilde{\phib}]=\0$ and
 $E[\Pi_{\phib}^\bot{\bar{\u}}]=E[1\cdot\Pi_{\phib}^\bot{\bar{\u}}]=\0$. 
 Hence, one has $E[\Pi_{\phib}{\bar{\u}}]=E[\bar{\u}]-E[\Pi_{\phib}^\bot{\bar{\u}}]=\0$. 
 Our goad is to find $\widetilde{\phib}$ which minimizes 
 $V[B\etab-\frac{n'}{n+n'}\bar{\u}]$ in \eqref{eqn:diff-bar_u-expression} 
 in the sense of positive definiteness. 
 The orthogonal decomposition leads to 
 \begin{align*}
  V[B\etab-\frac{n'}{n+n'}\bar{\u}]
  =
  V[B\phib-\frac{n'}{n+n'}\Pi_{\phib}\bar{\u}]+V[B\widetilde{\phib}-\frac{n'}{n+n'}\Pi_{\phib}^{\bot}\bar{\u}], 
 \end{align*}
 because of the orthogonality between $B\phib-\frac{n'}{n+n'}\Pi_{\phib}\bar{\u}$ and
 $B\widetilde{\phib}-\frac{n'}{n+n'}\Pi_{\phib}^{\bot}\bar{\u}$, 
 and the equality $E[B\widetilde{\phib}-\frac{n'}{n+n'}\Pi_{\phib}^\bot{\bar{\u}}]=\0$. 
 Hence, $\widetilde{\phib}$ satisfying 
 \begin{align*}
  B\widetilde{\phib}=\frac{n'}{n+n'}\Pi_{\phib}^{\bot}\bar{\u} 
 \end{align*}
 is an optimal choice. 
 Since the matrix $B$ is row full rank, a solution of the above equation is given by 
 \begin{align*}
  \widetilde{\phib}=\frac{n'}{n+n'}B^T(BB^T)^{-1}\Pi_{\phib}^{\bot}\bar{\u}. 
 \end{align*}
 We obtain the maximum improvement of $\mathrm{Diff}[\u]$ 
 by using the equalities 
 $V[\Pi_{\phib}\bar{\u}]=E[\Pi_{\phib}\bar{\u}(\Pi_{\phib}\bar{\u})^T]$
 and 
 $B\phib=E[\bar{\u}\phib^T]E[\phib\phib^T]^{-1}\phib=\Pi_{\phib}\bar{\u}$. 
 \end{proof}

Suppose that the optimal moment function $\etab=\phib+\widetilde{\phib}$ presented in 
Theorem \ref{theorem:optimal-moment-func} is used with the score function
$\u(x,y;\alphab)$. Then, the improvement \eqref{eqn:maximum-improvement-opt} is 
maximized when $E[\Pi_{\phib}\bar{\u}(\Pi_{\phib}\bar{\u})^T]$ is minimized. 
Hence, the model $w(x;\thetab)$ with the lower dimensional parameter $\thetab$ 
is preferable as long as the
assumption in Theorem \ref{theorem:optimal-moment-func} is satisfied. 
This is intuitively understandable, because the statistical perturbation 
of the density-ratio estimator is minimized, when the smallest model is employed. 

\begin{remark}
 Suppose that the basis functions, $\phi_1(x),\ldots,\phi_r(x)$, 
 are closely orthogonal to $\bar{\u}$, i.e., 
 $E[\bar{\u}\phib^T]$ is close to the null matrix. 
 Then, the improvement $\mathrm{Diff}[\u]$ is close to $\frac{n'}{n+n'}E[\bar{\u}\bar{\u}^T]$. 
 As a result, we have 
 $\sup_{\phib}\mathrm{Diff}[\u]=\frac{n'}{n+n'}E[\bar{\u}\bar{\u}^T]$ in which the
 supremum 
 is taken over the basis of the density-ratio model 
 satisfying the assumption in Theorem \ref{theorem:optimal-moment-func}. 
 However, the basis functions satisfying the exact equality $E[\bar{\u}\phib^T]=O$ is useless. 
 Because, the equality $E[\bar{\u}\phib^T]=O$ leads to $B=O$ and 
 thus, the equality \eqref{eqn:diff-bar_u-expression} is reduced to 
 \begin{align*}
  \mathrm{Diff}[\u]=\frac{n'}{n+n'}E[\bar{\u}\bar{\u}^T]-\frac{n+n'}{n'}V[\frac{n'}{n+n'}\bar{\u}]+o(1)=o(1). 
 \end{align*}
 This result implies that there is the singularity at the basis function $\phib$ such that
 $E[\bar{\u}\phib^T]=O$. 
\end{remark}

It is not practical to apply the optimal function $\etab(x;\thetab)$ defined by
\eqref{eqn:optimal-eta}. 
The optimal moment function depends on $\bar{\u}$, and one needs information 
on the probability $p(y|x)$ to obtain the explicit form of $\bar{\u}$. 
The estimation of $\bar{\u}$ needs non-parametric estimation, 
since the model misspecification of $\mathcal{M}$ is significant in our setup. 
Thus, we consider more practical estimator for the density ratio. 
Suppose that $\widetilde{\phib}=\0$ holds for the moment function $\etab(x;\thetab^*)$. 
For example, the optimal moment function \eqref{eqn:optimal-density-ratio-est} 
satisfies $\etab(x;\thetab^*)=\frac{n}{n+n'}\phib(x)$ at $\thetab=\thetab^*$, i.e., 
$\widetilde{\phib}=\0$. 
For the density-ratio model $w(x;\thetab)=\exp\{\phib(x)^T\thetab\}$ with 
$\phi_1(x)=1$ and the moment function satisfying $\etab(x;\thetab^*)=\phib(x)$, 
a brief calculation yields that 
\begin{align}
 \mathrm{Diff}[\u]=
 \frac{n'-n}{n'}E[{\Pi_{\bm\phi}}\bar{\u}({\Pi_{\bm\phi}}\bar{\u})^{T}]+o(1). 
 \label{eqn:maximum-improvement-tildephi0} 
\end{align}
Hence, the improvement is attained, when $n<n'$ holds. 
As an interesting fact, we see that the larger model $w(x;\thetab)$ attains the better
improvement in \eqref{eqn:maximum-improvement-tildephi0}. 
Indeed, $\Pi_{\bm\phi}\bar{\u}$ gets close to $\bar{\u}$, 
when the density-ratio model $w(x;\thetab)=\exp\{\thetab^T\phib(x)\}$ becomes large. 
Hence, the non-parametric estimation of the density-ratio may be a good choice to achieve
a large improvement for the estimation of the conditional probability. 
This is totally different from the case that the optimal $\widetilde{\phib}$ 
presented in Theorem \ref{theorem:optimal-moment-func} is used in 
the density-ratio estimation. 
The relation between 
$\mathrm{Diff}[\u]$ using the optimal $\widetilde{\phib}$ and $\mathrm{Diff}[\u]$ with
$\widetilde{\phib}=\0$ is illustrated in Figure \ref{fig:improvement}. 
In the limit of the dimension of $\thetab$, both variance matrices converge to 
$\frac{n'-n}{n'}E[\bar{\u}\bar{\u}^T]$ monotonically. 

\begin{figure}
 \centering 
 \includegraphics[scale=0.5]{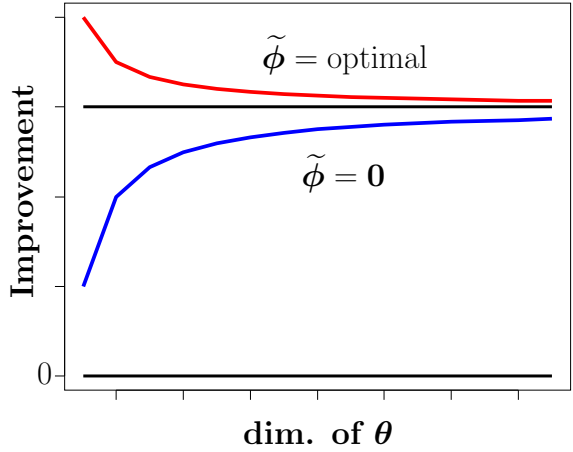}
 \caption{
 The improvement $\mathrm{Diff}[\u]$ is depicted as the function of the dimension
 of the density-ratio model. 
 Since the improvement is represented by the matrix, the figure gives 
 a view showing a frame format of the inequality relation. 
 When the dimension of $\thetab$ tends to infinity and $n'>n$ holds, 
 the two curves converges to the common positive definite matrix 
 $\frac{n'-n}{n'}E[\bar{\u}\bar{\u}^T]$. }
 \label{fig:improvement}
\end{figure}


\begin{example}
 Let $\u(x,y;\alphab)$ be the score function
 of the model $y=\alphab^T\b(x)+Z,\,Z\sim{}N(0,\sigma^2)$, 
 where $\b(x)=(b_1(x),\ldots,b_d(x))$ is the vector consisting of basis functions
 and $\sigma^2$ is a known parameter. 
 Then, one has $\u(x,y;\alphab)=(y-\alphab^T\b(x))\b(x)$. 
 Suppose that the true conditional probability leads to 
 the regression function $y=f(x)+Z$, where $E[Z|x]=0$ for all $x$. 
 Then, one has 
 $\bar{\u}(x;\alphab)=(f(x)-\alphab^T\b(x))\b(x)$
 and $E[\bar{\u}\bar{\u}^T]=E[(f(x)-\alphab^T\b(x))^2\b(x)\b(x)^T]$. 
 Hence, the upper bound of the improvement is governed by the 
 degree of the model misspecification $(f(x)-\alphab^T\b(x))^2$. 
 According to Theorem \ref{theorem:optimal-moment-func}, 
 an optimal moment function $\etab(x;\thetab)$ is given as 
 \begin{align*}
  \etab(x;\thetab^*)
  =
  \phib(x)+\frac{n'}{n+n'}B^T(BB^T)^{-1}
  \big((f(x)-{{\alphab}^*}^T\b(x))\b(x)-B\phib(x)\big)
 \end{align*}
 at $\thetab=\thetab^*$, where $B=E[(f-{\alphab^*}^T\b)\b\phib^T]E[\phib\phib^T]^{-1}$. 
\end{example}

\section{Numerical Experiments}
\label{sec:Numerical_Experiments}
We show numerical experiments to compare the standard supervised learning and 
the semi-supervised learning using DRESS. 
Both regression problems and classification problems are presented.

\subsection{Regression problems}
\label{subsec:Regression}
We consider the regression problem with the
$d$-dimensional covariate variable shown below. 
\begin{description}
 \item[labeled data:] 
	    \begin{align}
	     y_i &= \1^T\x_i+\varepsilon\frac{\|\x_i\|^2}{d}+z_i,
	     \  z_i\sim{}N(0,\sigma^2),\ i=1,\ldots,n,  \label{eqn:sim_reression_function}\\
	       x_i&\sim{}N_{d}(\0,I_{d}),\quad \1^T=(1,\ldots,1)\in\Rbb^{d}. \nonumber
	    \end{align}
 \item[unlabeled data:]
	    $\x_j'\sim{}N_{d}(\0,I_{d}),\,j=1,\ldots,n'$. 
 \item[regression model:]
	    $y=\alphab^T\x+z,\quad \alphab\in\Rbb^d,\quad z\sim{}N(0,s^2)$. 
 \item[score function:]
	    $\u(x,y;\alphab)=(y-\alphab^T\x)\x$. 
\end{description}

The parameter $\varepsilon$ in \eqref{eqn:sim_reression_function}
implies the degree of the model misspecification. 
Let $f_\varepsilon$ be the target function, 
$f_\varepsilon(x)=\1^T\x+\varepsilon\|\x\|^2/d$, and 
define 
\begin{align*}
 e(\varepsilon)=\min_{\bm \alpha} E_{\x}[|f_\varepsilon(\x)-\alphab^T\x|^2], 
\end{align*}
which implies the squared distance from the true function $f_\varepsilon$ to the 
linear regression model. 
On the other hand,
the mean square error of the naive least mean square (LMS) estimator
$\widetilde{\alphab}$, i.e., 
$E_{\mathrm{Data}}[E_{\x}[|f_0(\x)-\widetilde{\alphab}^T\x|^2]]$, 
is asymptotically equal to $\sigma^2d/n$, when the model is specified. 
We use the ratio
\begin{align*}
 \delta=\sqrt{e(\varepsilon)}\big/\sqrt{\frac{\sigma^2 d}{n}}
 =\sqrt{\frac{e(\varepsilon)n}{\sigma^2d}}
\end{align*}
as the normalized measure of the model misspecification. 
When $\delta\gg 1$ holds, the misspecification of the model can be
statistically detected. 

First, we use a parametric model for density ratio estimation. 
For any positive integer $k$, let $\x^{(k)}$ be the $d$-dimensional vector
$(x_1^k,\ldots,x_d^k)^T$. The density-ratio model is defined as 
\begin{align*}
 w(\x;\thetab)=\exp\left\{ \theta_0+\thetab_1^T\x+\thetab_2^T\x^{(2)}+\cdots+\thetab_L^T\x^{(L)} \right\}
\end{align*}
having $Ld+1$ dimensional parameter $(\theta_0,\thetab_1,\ldots,\thetab_L)$. 
We apply the estimator \eqref{eqn:optimal-density-ratio-est} presented by Qin
\cite{Biometrika:Qin:1998}. 
Note that the estimator 
\eqref{eqn:optimal-density-ratio-est} satisfies $\widetilde{\phib}=\0$ at
$\thetab=\thetab^*$. Hence, the improvement is asymptotically given by 
\eqref{eqn:maximum-improvement-tildephi0}. 
Under the setup of $d=2, n=500, n'=5000$ and $\sigma=0.2$, we compute 
the mean square errors for LMS estimator 
$\widetilde{\alphabold}$ and DRESS $\widehat{\alphabold}$. 
The difference of test errors, 
\begin{align*}
 n\cdot
 (E[(\widetilde{\alphabold}^T\x-f_\varepsilon(\x))^2]-E[(\widehat{\alphabold}^T\x-f_\varepsilon(\x))^2]), 
\end{align*}
is evaluated for each $\varepsilon$ and each dimension of the density ratio, $Ld+1$, 
where the expectation is evaluated over the test samples. 
The mean square error is calculated by the average over 500 iterations. 

Figure \ref{fig:sim_result_parametric_DRmodel} shows the results. 
When the model is specified, i.e., $\delta=0\,(\varepsilon=0)$, 
LMS estimator presents better performance than DRESS. 
Under the practical setup such as $\delta>1$, 
however, we see that DRESS outperforms LMS estimator. 
The dependency on the dimension of the density-ratio model is not clearly detected in this
experiment. 
Overall, larger density-ratio model presents rather unstable result. 
Indeed, in DRESS with large density ratio model, say the right bottom panel 
in Figure \ref{fig:sim_result_parametric_DRmodel}, 
the mean square error of DRESS can be large, i.e., the improvement is negative, 
even when the model misspecification $\delta$ is large. 

\begin{figure}[tb]
 \begin{center}
  \hspace*{-7mm}
  \begin{tabular}{ccc}
   \includegraphics[scale=0.3]{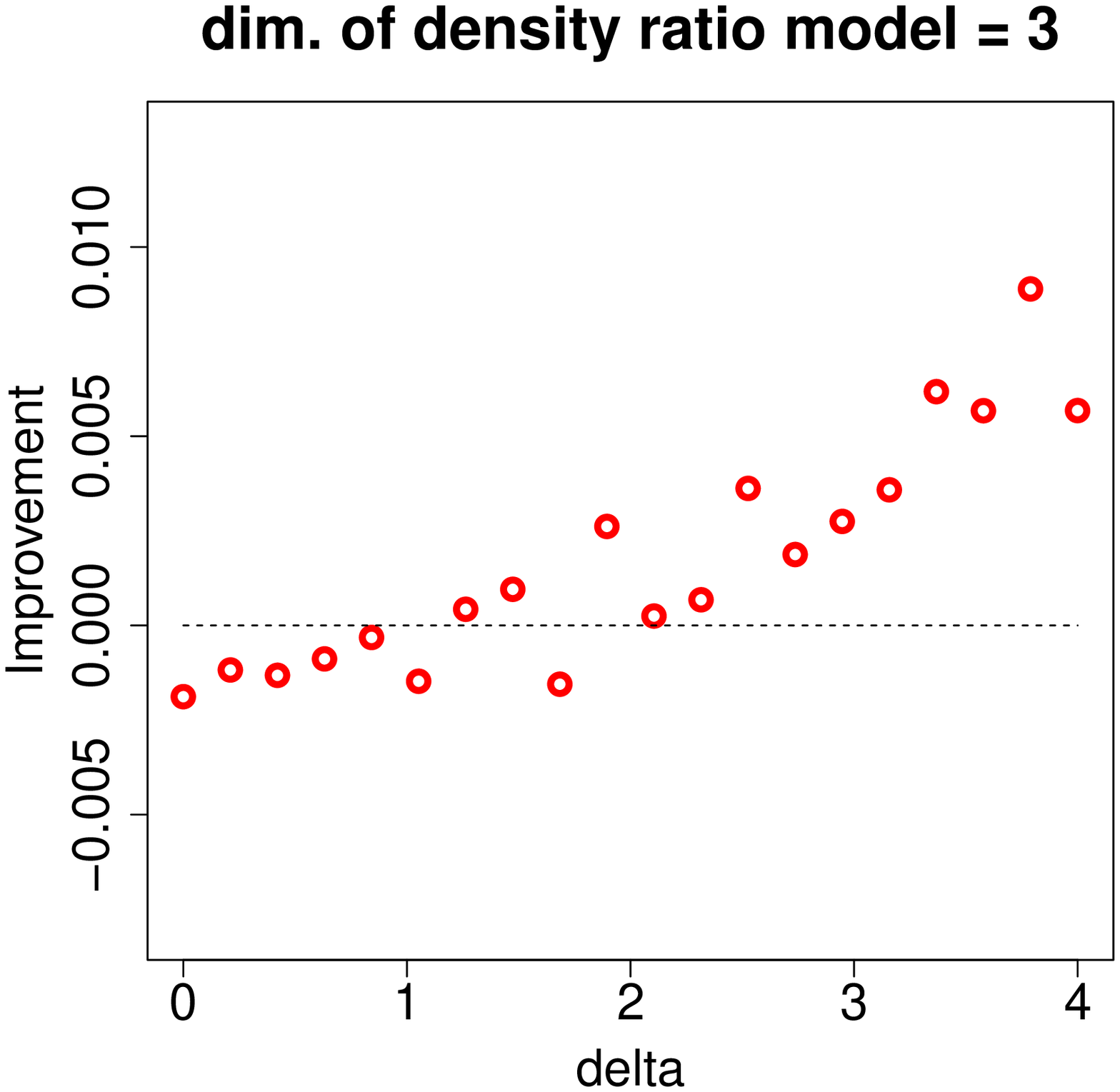}  &
   \includegraphics[scale=0.3]{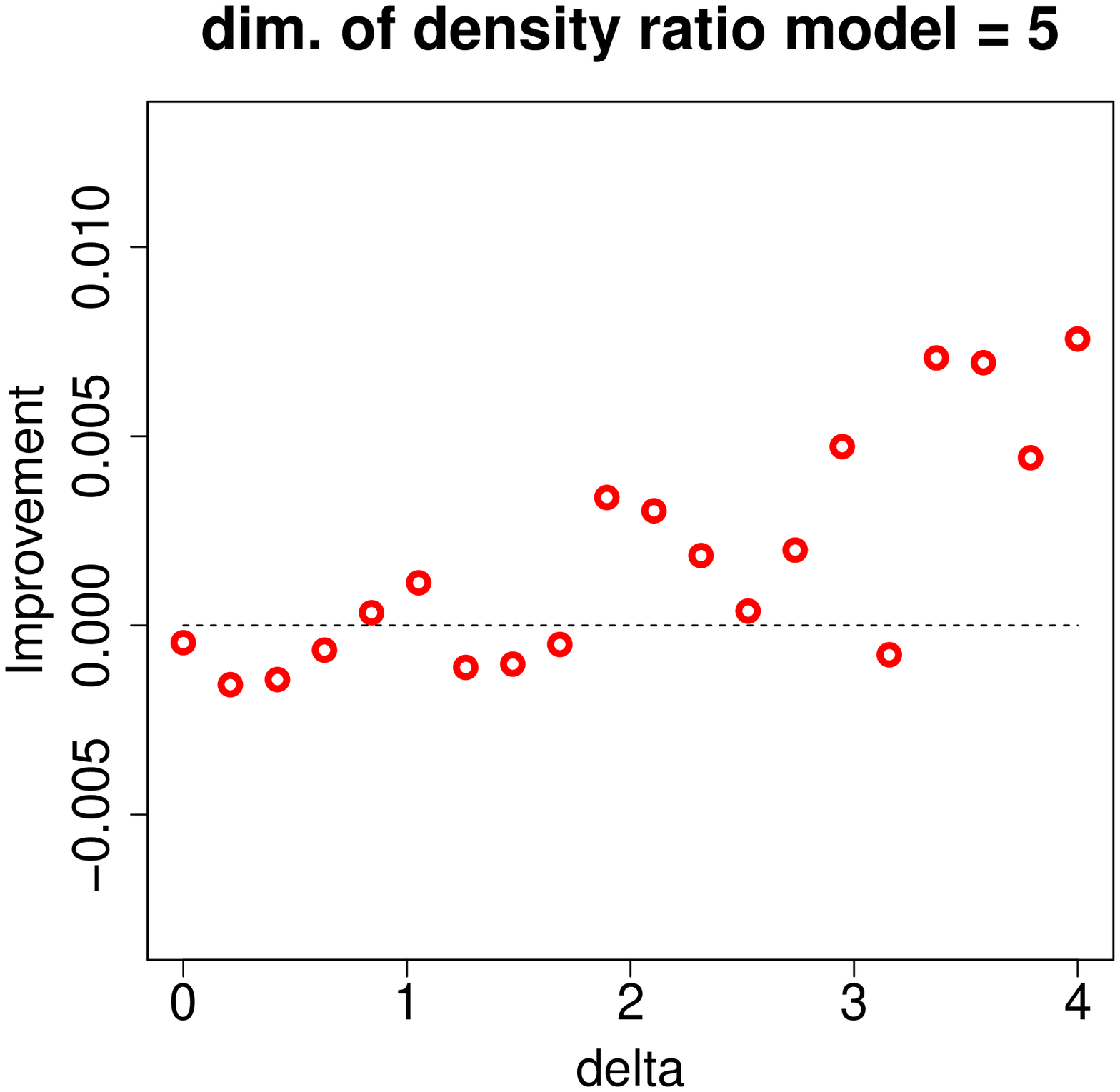}  &
   \includegraphics[scale=0.3]{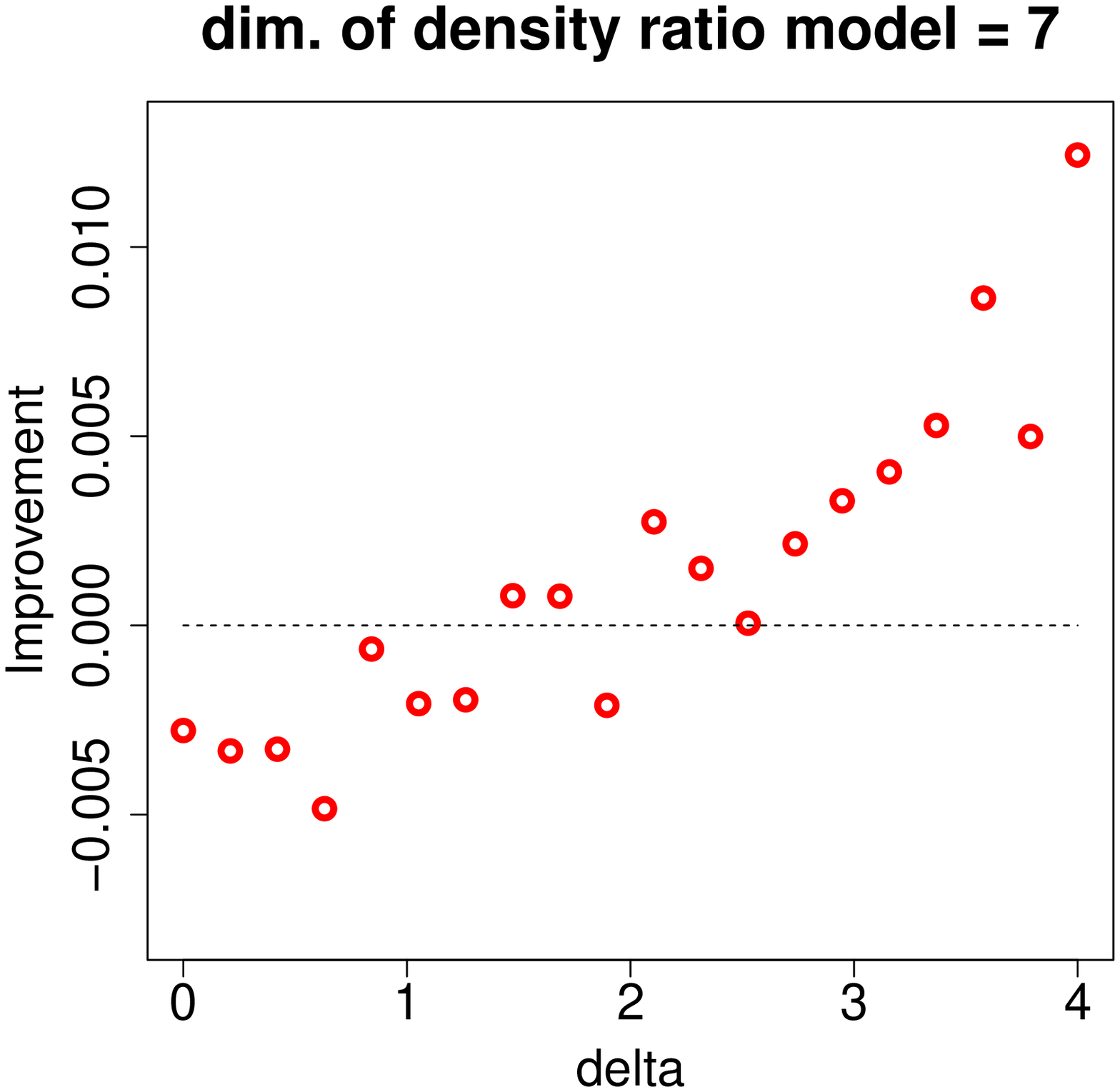}  \\
   \includegraphics[scale=0.3]{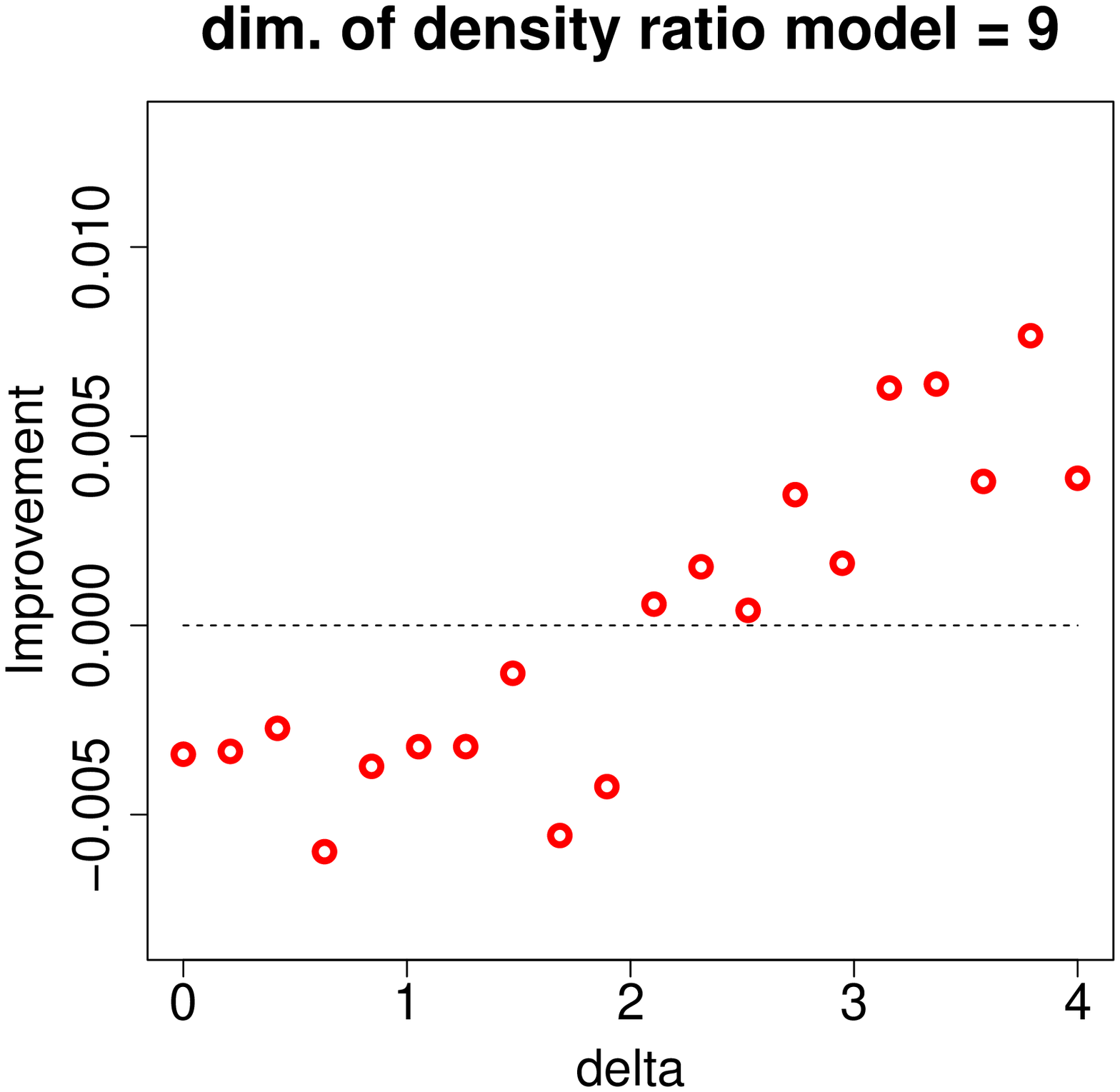}  &
   \includegraphics[scale=0.3]{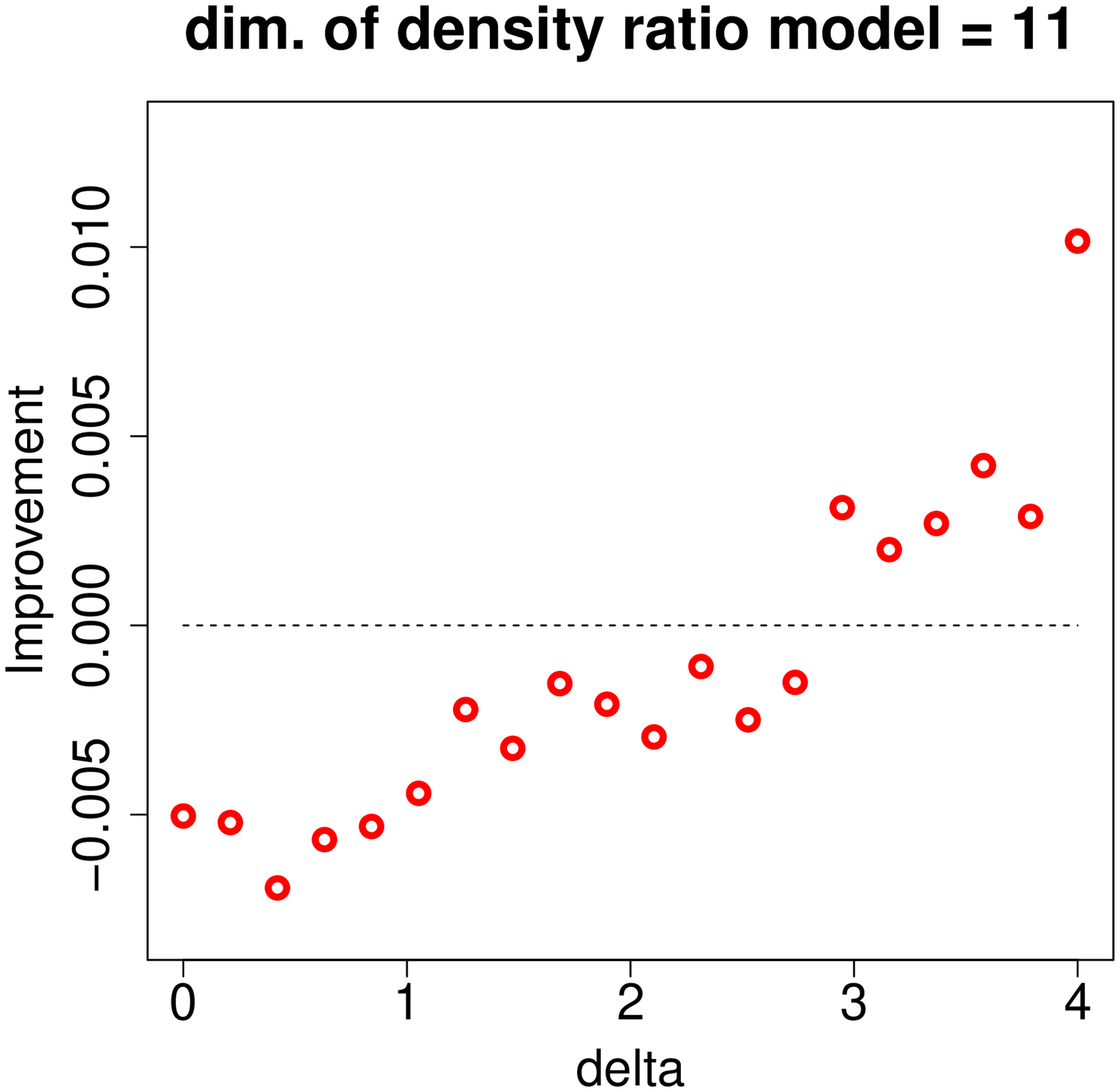} &
   \includegraphics[scale=0.3]{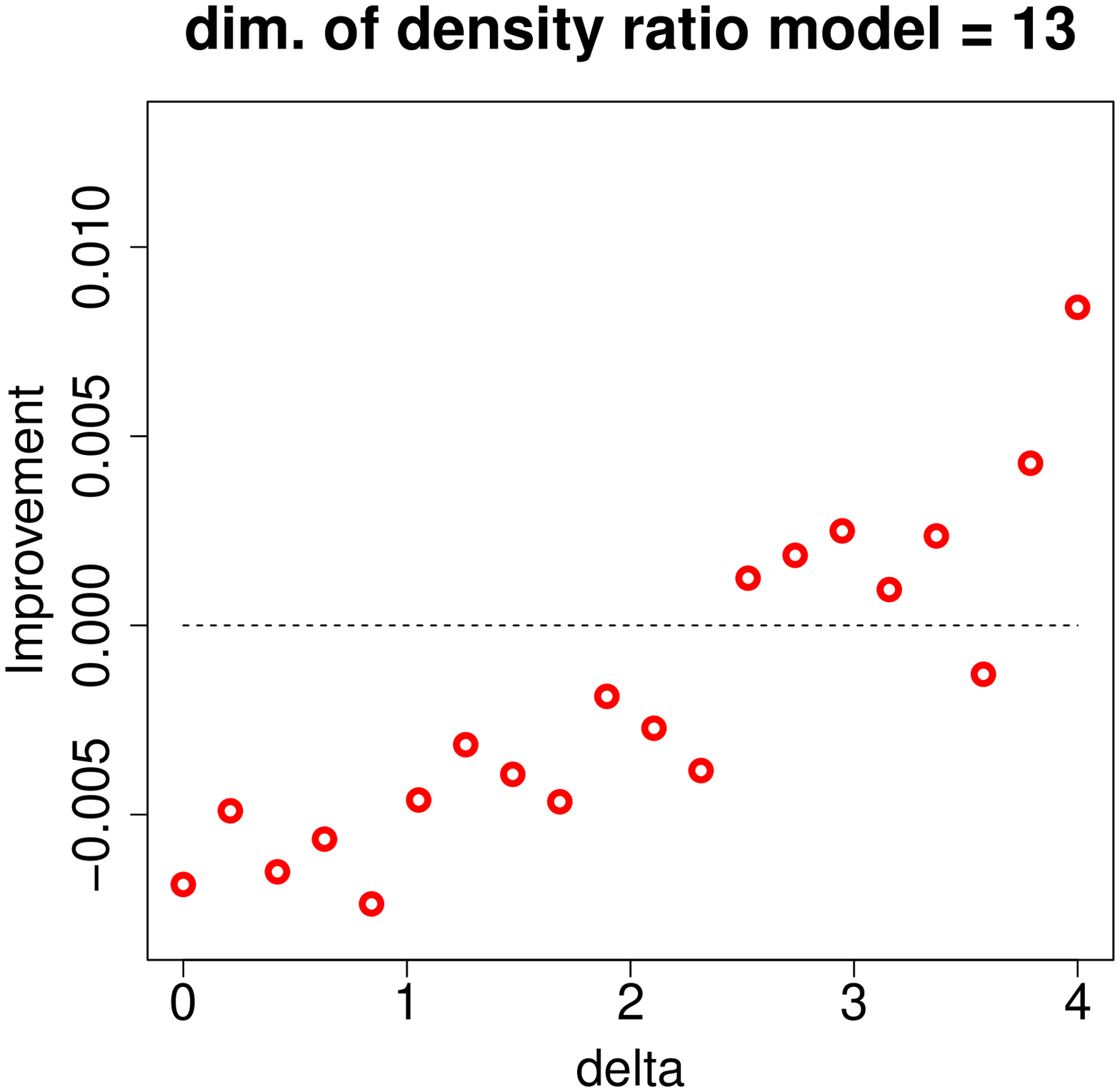} \\
  \end{tabular}
  \caption{The difference of the mean square errors is plotted as the function of $\delta$, 
  where $\delta$ is the normalized measure of the model misspecification.
  The vertical axes ``Improvement'' denotes the difference of the mean square errors 
  between LMS estimator and DRESS. Positive improvement denotes that DRESS
  outperforms LMS estimator. }
  \label{fig:sim_result_parametric_DRmodel}
 \end{center}
\end{figure}

Next, we compare LMS estimator and DRESS with a nonparametric estimator of the
density-ratio. 
Here, we use KuLSIF \cite{kanamori12:_statis} as the density-ratio estimator. 
KuLSIF is a non-parametric estimator of the density-ratio based on the kernel method. 
The regularization is efficiently conducted to suppress the degree of freedom of the
nonparametric model. In KuLSIF, the kernel function of the reproducing kernel Hilbert
space corresponds to the basis function $\phib(x)$. 


Under the setup of $d=10, n=50, n'=20,100,1000$ and $\sigma=0.1,0,2,0.5$, 
we compute the mean square errors by the average over 100 iterations. 
In Figure \ref{fig:numerical_experiments}, the square root of the mean square errors 
for LMS estimator and DRESS are plotted as the function of $\delta$, 
i.e., (model error)/(statistical error). 
When $\delta$ is around $1$, it is statistically hard to detect the model misspecification
by the training data of the size $n=50$.  
When the model is specified ($\varepsilon=0$), LMS estimator presents better 
performance than DRESS. 
Under the practical setup such as $\delta>1$, 
however, we see that DRESS with KuLSIF outperforms LMS estimator. 
As shown in the asymptotic analysis, we notice that the sample size of the unlabeled data
affects the estimation accuracy of DRESS. 
The numerical results show that DRESS with large $n'$ attains the smaller error 
comparing to DRESS with small $n'$, especially when $\delta>1$ holds. 
In the numerical experiment, even DRESS with $n=50$ and $n'=20$ slightly outperforms LMS
estimator. This is not supported by the asymptotic analysis.  
Hence, we need more involved theoretical study about the statistical feature of
semi-supervised learning.

\begin{figure}[tb]
 \begin{center}
  \hspace*{-6mm}
  \begin{tabular}{ccc}
 \includegraphics[scale=0.3]{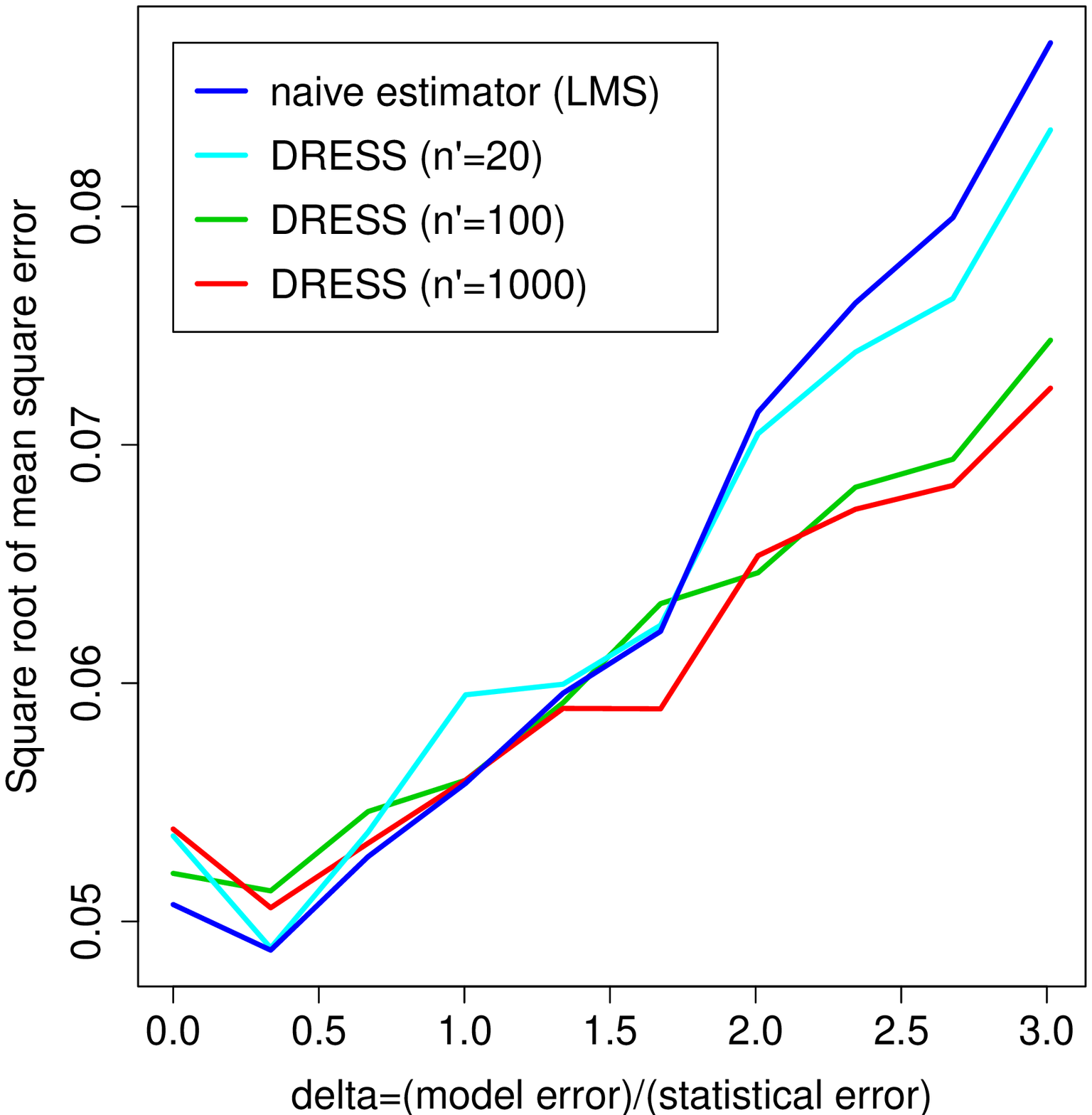} &
 \includegraphics[scale=0.3]{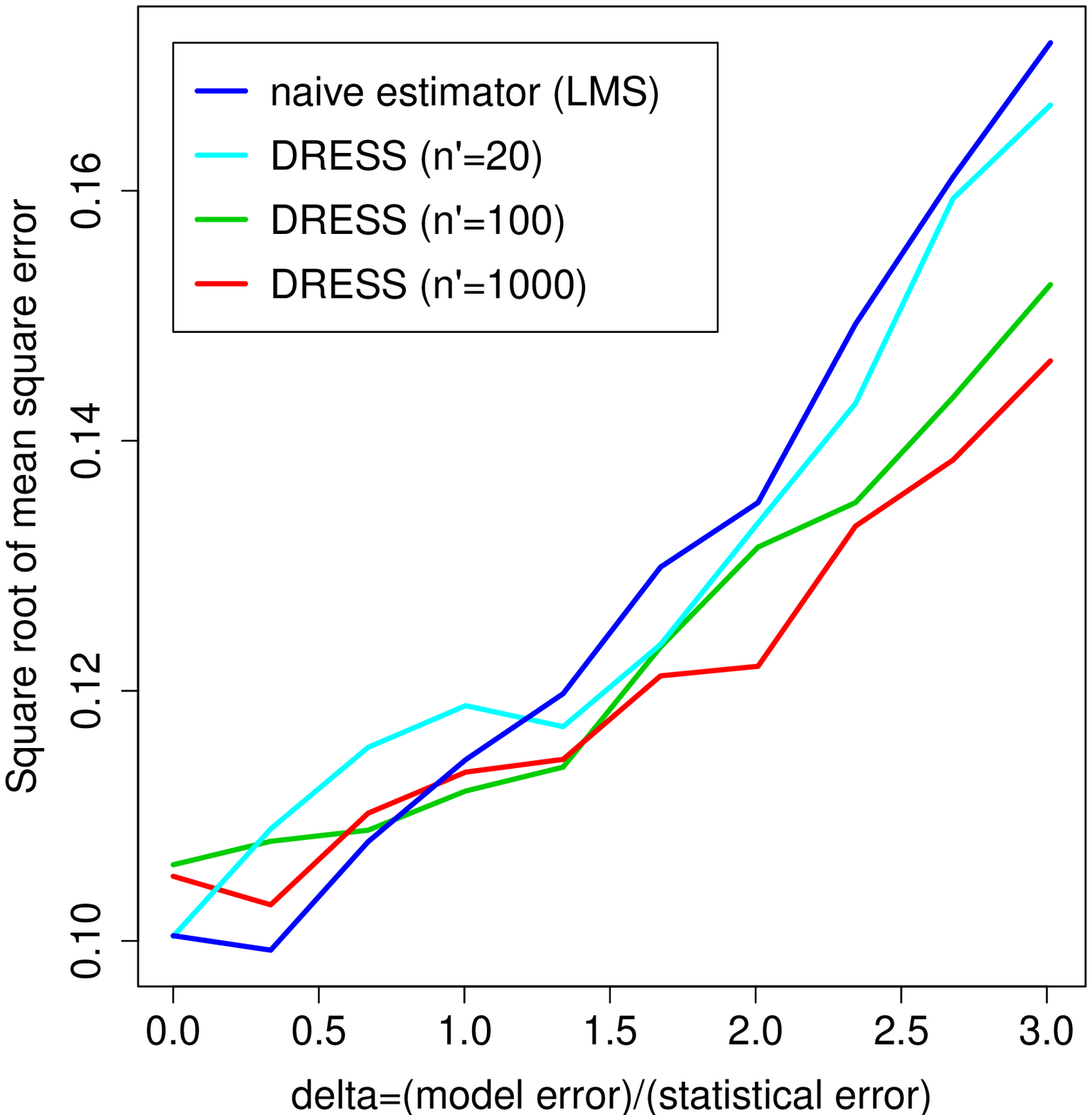} &
 \includegraphics[scale=0.3]{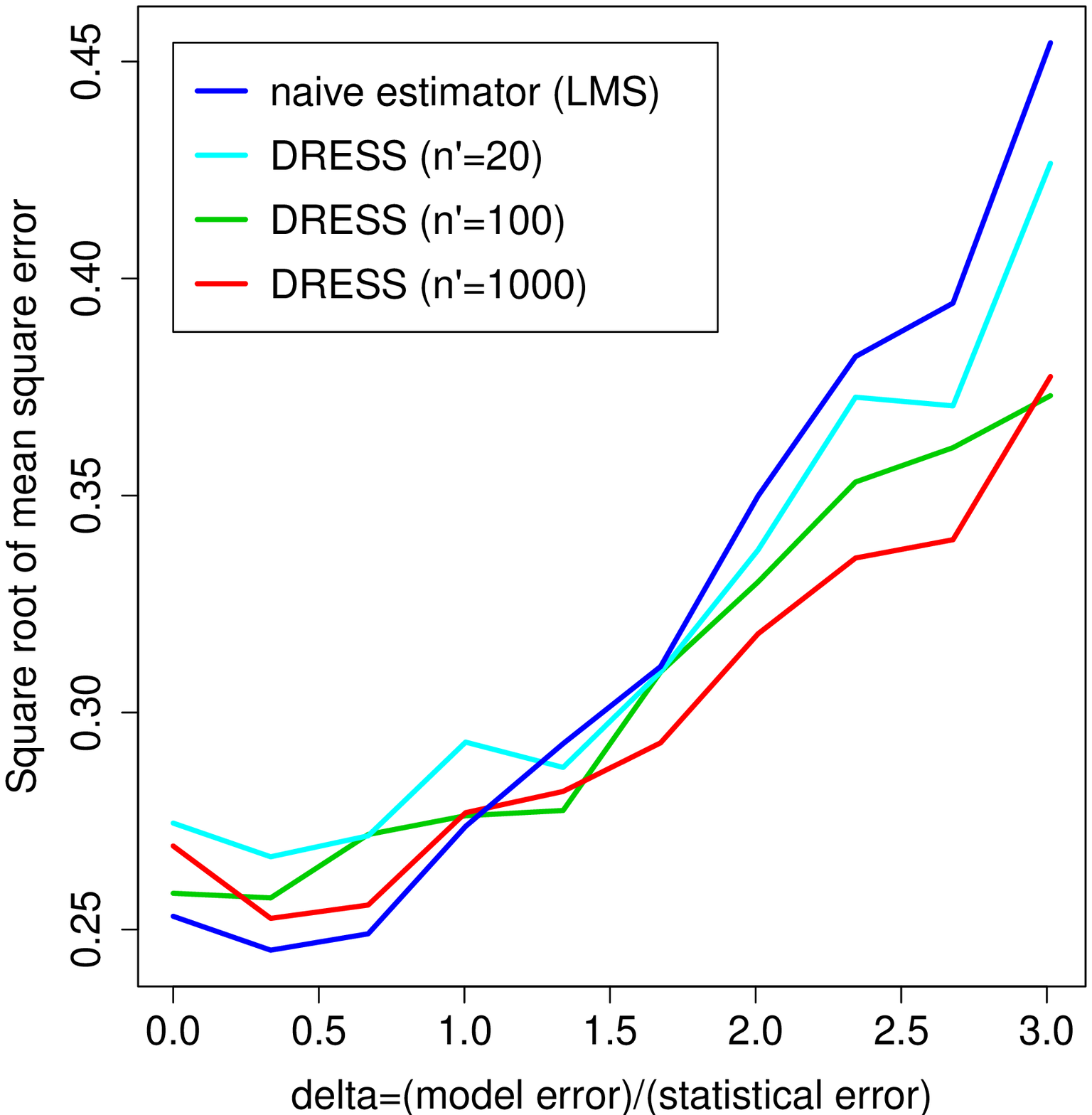} \\ 
   $\sigma=0.1$ & $\sigma=0.2$ & $\sigma=0.5$  
  \end{tabular}
  \caption{The square root of mean square errors of naive estimator and DRESS 
  with $n'=20,100,1000$ are depicted as the function of $\delta$, 
  where $\delta$ is the normalized measure of the model misspecification. 
  The sample size of the labeled data is $n=50$, and $\sigma$ is the standard deviation of 
  the noise involved in the dependent variable $y$. }
  \label{fig:numerical_experiments}
 \end{center}
\end{figure}


\subsection{Classification problems}
\label{subsec:Classification}
As a classification task, 
we use {\tt spam} dataset in ``kernlab'' of R package \cite{karatzoglou04:_kernl}. 
The dataset includes 4601 samples. The dimension of the covariate is 57, i.e.,
$\x=(x_1,\ldots,x_{57})^T$ whose elements represent statistical features of each
document. The output $y$ is assigned to ``spam'' or ``nonspam''. 

For the binary classification problem, we use the logistic model, 
\begin{align*}
 P(\text{spam}\,|\,\x;\alphab)=\frac{1}{1+\exp\{-\alpha_0-\sum_{d=1}^D{\alpha_d}x_d\}}, 
\end{align*}
where $D$ is the dimension of the covariate used in the logistic model. 
In numerical experiments, $D$ varies from 10 to 57, hence, 
the dimension of the model parameter $\alphab$ varies from 11 to 58. 
We tested DRESS with KuLSIF \cite{kanamori12:_statis}
and MLE with $n=200, 500, 800$ randomly chosen 
labeled training samples and $n'=100,500,1000,2000$ unlabeled 
training samples. The remaining samples are served as the test data. 
The score function $\u(x,y;\alphab)=\nabla\log{P}(y|\x;\alphab)$ is used for the
estimation. 

Table \ref{table:spam_error} shows the prediction errors $(\%)$  
with the standard deviation. 
We also show the p-value of the one-tailed paired $t$-test 
for prediction errors of DRESS and MLE. Small p-values denote the superiority of DRESS. 
We notice that p-value is small when the dimension $D$ is not large. 
In other word, the numerical results meet the asymptotic theory in Section
\ref{sec:Maximum_Improvement_SSL}. 
For relatively high dimensional models, 
the prediction error of MLE is smaller than that of DRESS; see the row of $D=57$ in Table
\ref{table:spam_error}. 
The size of unlabeled data, $n'$, also affects the results. 
Indeed, the p-value becomes small for large $n'$. 
This result is supported by the asymptotic analysis presented in Section
\ref{sec:Maximum_Improvement_SSL}. 

\begin{table}[htb]
 \caption{Prediction errors $(\%)$ for DRESS with KuLSIF and MLE are shown. 
 The p-values of the one-tailed paired $t$-test for prediction errors are also
 presented.}
 \label{table:spam_error}
 \begin{center}
  \hspace*{-10mm}
  \footnotesize
  \begin{tabular}{cccccccccccc}
   & \multicolumn{3}{c}{$n=200,\,n'=100$} && \multicolumn{3}{c}{$n=500,\,n'=100$} && \multicolumn{3}{c}{$n=800,\,n'=100$} \\
   \cline{2-4}  \cline{6-8}  \cline{10-12} 
   $D$&        DRESS &          MLE &p-value&&       DRESS &           MLE &p-value&&        DRESS &           MLE &p-value \\ \hline
   10&21.48$\pm$0.95&21.69$\pm$1.09&0.023&&20.86$\pm$0.76&20.93$\pm$0.82&0.163&&20.73$\pm$0.71&20.72$\pm$0.67&0.541\\	
   20&18.81$\pm$1.30&18.54$\pm$1.42&0.987&&17.15$\pm$0.79&17.16$\pm$0.90&0.424&&16.63$\pm$0.68&16.93$\pm$0.87&0.000\\	
   30&14.67$\pm$1.54&14.44$\pm$1.50&0.993&&11.83$\pm$0.75&11.92$\pm$0.83&0.056&&11.29$\pm$0.51&11.39$\pm$0.56&0.057\\	
   40&16.16$\pm$1.79&16.06$\pm$1.83&0.910&&12.18$\pm$0.81&12.19$\pm$0.84&0.410&&11.24$\pm$0.60&11.40$\pm$0.62&0.005\\	
   50&15.98$\pm$2.49&15.84$\pm$2.45&0.939&&11.41$\pm$0.94&11.25$\pm$0.94&0.988&&10.13$\pm$0.66&10.15$\pm$0.65&0.359\\	
   57&15.08$\pm$2.64&15.01$\pm$2.67&0.777&&10.83$\pm$1.06&10.59$\pm$0.90&1.000&& 9.07$\pm$0.61& 8.98$\pm$0.70&0.959\vspace*{2mm}\\
   & \multicolumn{3}{c}{$n=200,\,n'=500$} && \multicolumn{3}{c}{$n=500,\,n'=500$} && \multicolumn{3}{c}{$n=800,\,n'=500$} \\
   \cline{2-4}  \cline{6-8}  \cline{10-12} 
   $D$&        DRESS &          MLE &p-value&&       DRESS &           MLE &p-value&&        DRESS &           MLE &p-value \\ \hline
   10&21.34$\pm$0.88&21.59$\pm$1.12&0.003&&20.56$\pm$0.70&21.06$\pm$0.84&0.000&&20.40$\pm$0.62&20.76$\pm$0.71&0.000\\
   20&18.58$\pm$1.40&18.60$\pm$1.45&0.406&&16.76$\pm$0.79&17.10$\pm$0.96&0.000&&16.51$\pm$0.67&16.95$\pm$0.90&0.000\\
   30&14.46$\pm$1.50&14.48$\pm$1.39&0.392&&11.71$\pm$0.70&11.86$\pm$0.73&0.002&&11.21$\pm$0.56&11.46$\pm$0.58&0.000\\
   40&15.88$\pm$1.98&15.83$\pm$2.05&0.759&&11.96$\pm$0.79&12.04$\pm$0.77&0.035&&11.13$\pm$0.56&11.41$\pm$0.62&0.000\\
   50&16.18$\pm$2.31&16.22$\pm$2.30&0.303&&11.24$\pm$0.92&11.26$\pm$0.93&0.350&&10.04$\pm$0.66&10.13$\pm$0.70&0.021\\
   57&14.88$\pm$2.82&14.77$\pm$2.77&0.933&&10.83$\pm$1.07&10.61$\pm$0.98&1.000&& 8.79$\pm$0.70& 8.81$\pm$0.63&0.319\vspace*{2mm}\\
   & \multicolumn{3}{c}{$n=200,\,n'=1000$} && \multicolumn{3}{c}{$n=500,\,n'=1000$} && \multicolumn{3}{c}{$n=800,\,n'=1000$} \\
   \cline{2-4}  \cline{6-8}  \cline{10-12} 
   $D$&        DRESS &          MLE &p-value&&       DRESS &           MLE &p-value&&        DRESS &           MLE &p-value \\ \hline
   10&21.26$\pm$0.96&21.74$\pm$1.28&0.000&&20.57$\pm$0.74&21.02$\pm$0.80&0.000&&20.29$\pm$0.61&20.74$\pm$0.64&0.000\\
   20&18.37$\pm$1.27&18.63$\pm$1.45&0.001&&16.78$\pm$0.70&17.08$\pm$1.00&0.000&&16.47$\pm$0.65&16.93$\pm$0.80&0.000\\
   30&14.53$\pm$1.51&14.60$\pm$1.42&0.089&&11.73$\pm$0.67&12.04$\pm$0.78&0.000&&11.16$\pm$0.62&11.43$\pm$0.69&0.000\\
   40&16.05$\pm$1.97&16.06$\pm$1.92&0.463&&11.84$\pm$0.78&11.91$\pm$0.75&0.098&&11.19$\pm$0.63&11.45$\pm$0.72&0.000\\
   50&15.58$\pm$2.16&15.52$\pm$2.10&0.703&&11.20$\pm$0.86&11.20$\pm$0.86&0.566&& 9.94$\pm$0.75&10.06$\pm$0.80&0.006\\
   57&14.99$\pm$2.86&14.94$\pm$2.93&0.684&&10.83$\pm$1.04&10.75$\pm$0.98&0.935&& 8.88$\pm$0.72& 8.99$\pm$0.73&0.014\vspace*{2mm}\\
   & \multicolumn{3}{c}{$n=200,\,n'=2000$} && \multicolumn{3}{c}{$n=500,\,n'=2000$} && \multicolumn{3}{c}{$n=800,\,n'=2000$} \\
   \cline{2-4}  \cline{6-8}  \cline{10-12} 
   $D$&        DRESS &          MLE &p-value&&       DRESS &           MLE &p-value&&        DRESS &           MLE &p-value \\ \hline
   10&21.31$\pm$1.06&21.78$\pm$1.27&0.000&&20.49$\pm$0.81&21.00$\pm$0.94&0.000&&20.18$\pm$0.85&20.70$\pm$1.02&0.000\\
   20&18.36$\pm$1.35&18.62$\pm$1.51&0.008&&16.79$\pm$0.86&17.18$\pm$1.09&0.000&&16.37$\pm$0.80&16.88$\pm$0.97&0.000\\
   30&14.66$\pm$1.71&14.53$\pm$1.69&0.956&&11.65$\pm$0.77&11.82$\pm$0.79&0.001&&11.12$\pm$0.79&11.44$\pm$0.85&0.000\\
   40&15.78$\pm$1.76&15.60$\pm$1.74&0.985&&11.81$\pm$0.90&12.10$\pm$0.97&0.000&&10.94$\pm$0.81&11.33$\pm$0.79&0.000\\
   50&16.21$\pm$2.17&16.01$\pm$2.14&0.973&&11.24$\pm$1.02&11.29$\pm$0.98&0.183&&10.01$\pm$0.78&10.19$\pm$0.78&0.001\\
   57&14.87$\pm$2.57&14.95$\pm$2.55&0.170&&10.52$\pm$1.13&10.56$\pm$1.14&0.187&& 8.71$\pm$0.79& 8.91$\pm$0.84&0.000\vspace*{2mm}\\
  \end{tabular}
 \end{center}
\end{table}


\section{Conclusion}
\label{sec:Conclusion}
In this paper, we investigated the semi-supervised learning with density-ratio estimator. 
We proved that the unlabeled data is useful when the model of the conditional probability
$p(y|x)$ is misspecified. This result agrees to the result given by 
Sokolovska, et al.~\cite{sokolovska08}, 
in which the weight function is estimated by using the estimator 
of the marginal probability $p(x)$ under a specified model of $p(x)$. 
The estimator proposed in this paper is useful in practice, 
since our method does not require the well-specified model for the marginal probability. 
Numerical experiments present the effectiveness of our method. 
We are currently investigating semi-supervised learning from the perspective of
semiparametric inference with missing data. 
A positive use of the statistical paradox in semiparametric inference is 
an interesting future work for semi-supervised learning.

\section*{Acknowledgement}
The authors are grateful to Dr.~Masayuki Henmi, Dr.~Hironori Fujisawa and Prof.~Shinto
Eguchi of Institute of Statistical Mathematics. 
TK was partially supported by Grant-in-Aid for Young Scientists (20700251).

\bibliographystyle{mlapa}




\end{document}